%% file: example_paper.tex
\documentclass[11pt]{article}
\usepackage[ruled,vlined]{algorithm2e} 
\usepackage{multirow}
\usepackage{times}
\usepackage{amsmath,amssymb}
\usepackage{amsthm}
\usepackage{graphicx}
\usepackage{booktabs}
\usepackage{hyperref}
\usepackage[capitalize,noabbrev]{cleveref}
\usepackage{textcomp}
\usepackage{wrapfig}
\usepackage{graphicx}
\usepackage{subcaption} 
\usepackage[textsize=tiny]{todonotes}
\allowdisplaybreaks[1] 

\usepackage{rotating}
\usepackage{natbib} 
\title{Toward Faithful and Complete Answer Construction from a Single Document}

\author{%
Zhaoyang Chen \\
Department of Mechanical Engineering \\
Iowa State University \\
2529 Union Drive, Ames, IA 50011-2030 \\
\texttt{zchen1@iastate.edu}
\and
Cody Fleming \\
Department of Mechanical Engineering \\
Iowa State University \\
2529 Union Drive, Ames, IA 50011-2030 \\
\texttt{flemingc@iastate.edu}
}

\date{}

\theoremstyle{plain}
\newtheorem{claim}{Claim}

\theoremstyle{definition}

\theoremstyle{remark}

\counterwithout{figure}{section}

\begin{document}
\maketitle

\begin{abstract}
Modern large language models (LLMs) are powerful generators driven by statistical next-token prediction.
While effective at producing fluent text, this design biases models toward high-probability continuations
rather than exhaustive and faithful answers grounded in source content.
As a result, directly applying LLMs lacks systematic mechanisms to ensure both completeness
(avoiding omissions) and faithfulness (avoiding unsupported content), which fundamentally conflicts
with core AI safety principles.
To address this limitation, we present EVE, a structured framework for document-grounded reasoning.

Unlike free-form prompting, EVE constrains generation to a structured, verifiable pipeline that decomposes
high-rigor reasoning into extraction, validation, and enumeration.
Empirically, this design enables consistent and simultaneous improvements in recall, precision, and F1-score:
recall and precision increase by up to 24\% and 29\%, respectively, with a corresponding 31\% gain in F1-score.
This effectively breaks the long-standing trade-off between coverage and accuracy typical of single-pass
LLM generation, while also mitigating generation truncation caused by length limitations.
At the same time, we emphasize that EVE exhibits performance saturation due to the inherent ambiguity of
natural language, reflecting fundamental limits of language-based reasoning.

\end{abstract}
\input{introduction}
\input{relatedWork}
\input{theory1}

\input{thoery2}

\input{STPA}
\input{limitations}

\bibliography{example_paper}
\bibliographystyle{plainnat}
\input{appendix}
\end{document}

%% file: introduction.tex
\section{Introduction}

\begin{figure}[htbp]
  \centering
  \begin{subfigure}{0.8\linewidth}
    \centering
    \includegraphics[width=\linewidth]{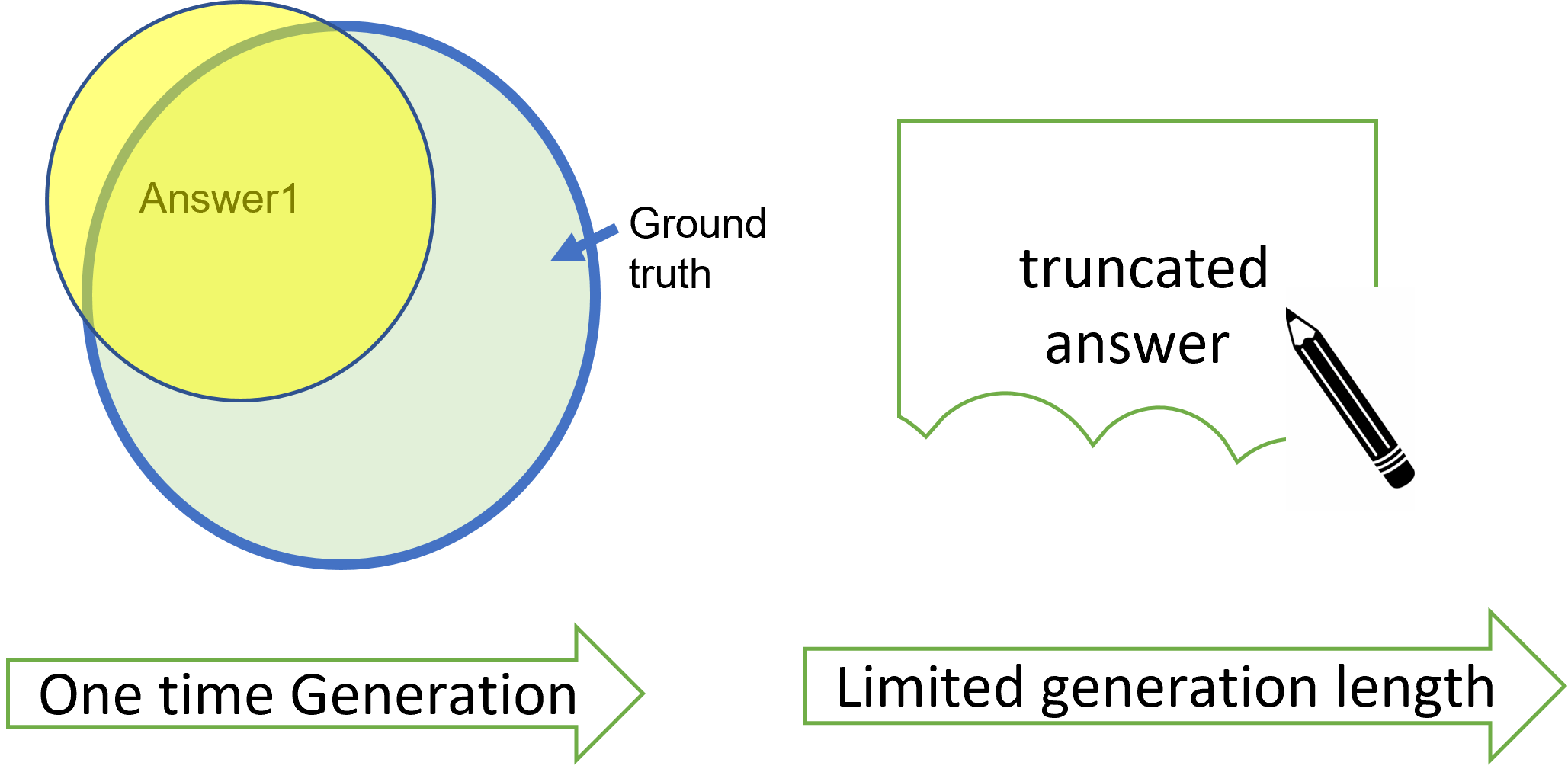}
    \caption{Illustration of traditional content-based question answering}
    \label{Illustration of traditional content-based question answering}
  \end{subfigure}
  \hfill
  \begin{subfigure}{0.8\linewidth}
    \centering
    \includegraphics[width=\linewidth]{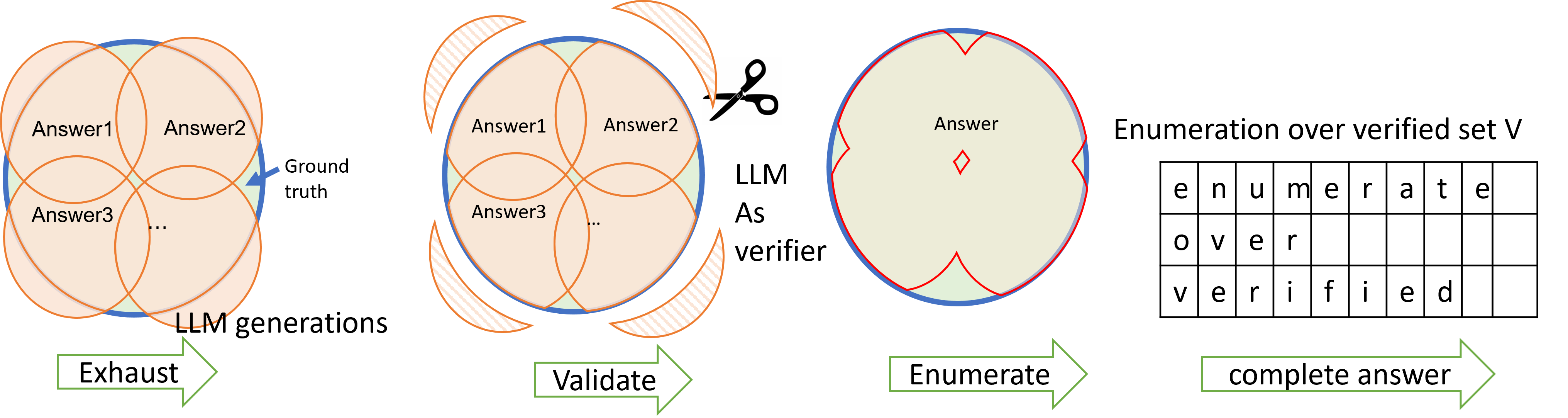}
    \caption{Illustration of EVE}
    \label{Illustration of EVE}
  \end{subfigure}
  \caption{comparison between traditional generation method and EVE} 
  \label{fig:my_figure} 
\end{figure}
A large fraction of real-world LLM applications operate in a single-document setting. The model is given one moderately long piece of text—such as a paper or technical specification—and is asked to extract information and answer questions based on that document. Within this broad class of tasks,  many tasks require a high level of rigor—such as safety analysis, legal auditing, and  autonomous system validation—demand not only plausible outputs, but strict guarantees on completeness and correctness. In these domains, missing a single critical element or introducing an unsupported statement can lead to severe real-world consequences. However, modern LLMs remain fundamentally likelihood-driven and end-to-end, generating outputs via in-distribution sampling. While such models achieve strong average-case performance, they offer no principled mechanism to bound the probability of correctness—a fundamental conflict with the reliability, robustness, and verifiability requirements central to AI safety.

In this work, we introduce Extraction–Validation–Enumerate (EVE), a modality-agnostic structured reasoning framework that replaces unconstrained generation, as shown in Figure~\ref{Illustration of traditional content-based question answering} and Figure~\ref{Illustration of EVE}.  By decomposing reasoning into element-wise search and validation stages—each with its own voting and aggregation mechanism—EVE transforms a single high-variance generative decision into a composition of low-variance, independently verifiable steps. We provide a theoretical analysis showing that this multi-stage architecture achieves exponential reduction in omission and hallucination probabilities. As an example, we include Figure~\ref{recall plot w.r.t number of independent queries} as an illustrative, high-level visualization to highlight the rapid reduction of omission errors under this process of multiple, independent querying and assimilation that results from EVE.

To demonstrate the generality and rigor of EVE, we apply it to a high-stakes safety-critical domain: System-Theoretic Process Analysis (STPA) \cite{leveson2011engineering,fleming2015integrating}. We introduce the first structured STPA dataset and present the first automated STPA analysis pipeline powered by LLMs, thereby expanding LLM-based reasoning into a domain where correctness is mandatory rather than optional. Across multiple state-of-the-art models, EVE delivers substantially more reliable, high-recall, and hallucination-resistant reasoning, effectively converting unbounded generative risk into bounded and quantifiable error probabilities. 


We also clarify the scope of applicability of our framework. Our approach targets single-document, closed-world settings where answers ideally remain tightly grounded in the source text and require exhaustive identification of key document-level elements. It is not designed for tasks where a short excerpt merely serves as a prompt for long, unconstrained internal reasoning followed by a synthesized answer, which are more naturally addressed by Chain-of-Thought or Tree-of-Thought style methods. In settings outside this scope, the framework is not expected to offer additional advantages over standard generation.
\begin{figure}
    \centering
    \includegraphics[width=1\linewidth]{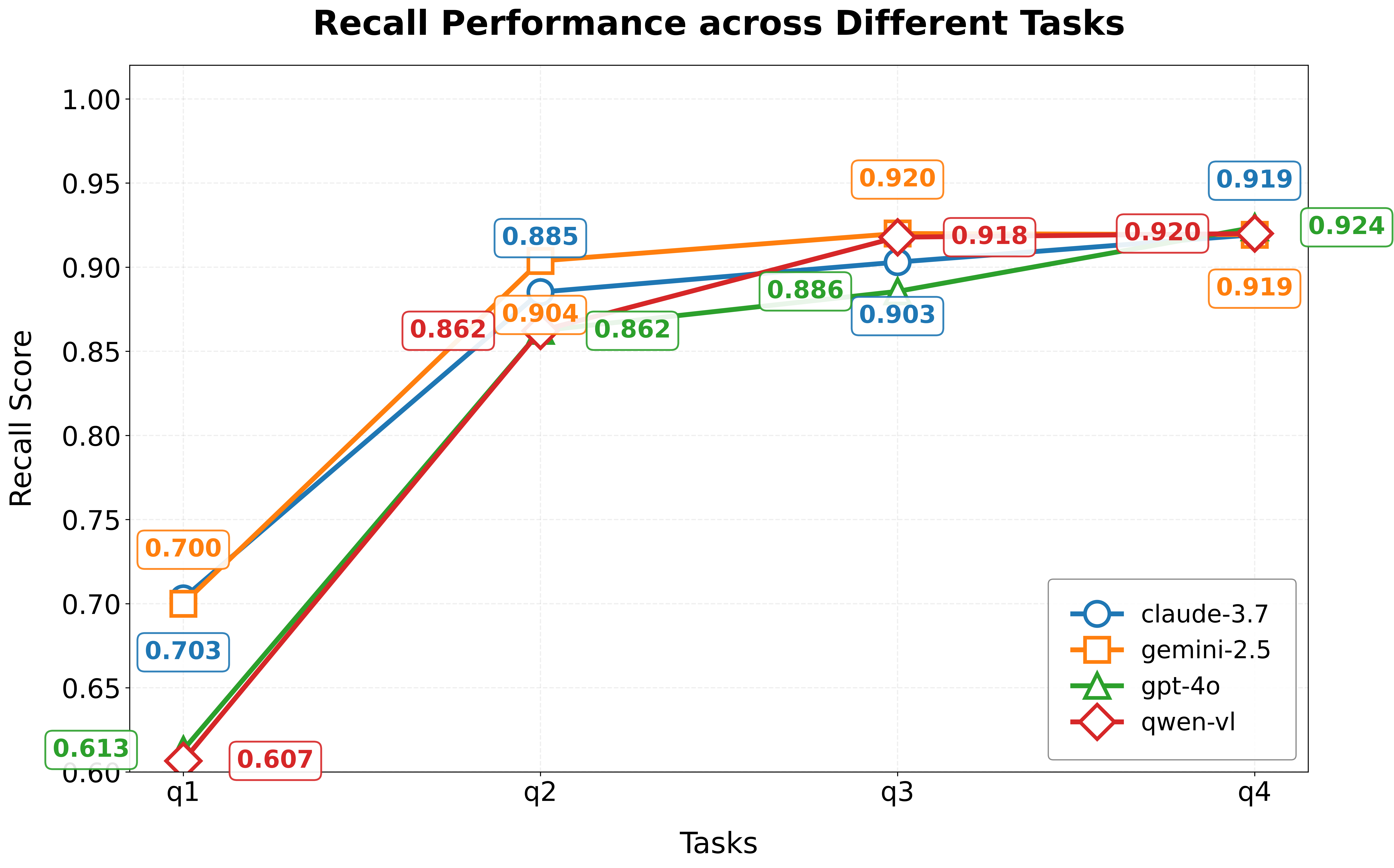}    
    \caption{recall plot w.r.t number of independent queries}
     \label{recall plot w.r.t number of independent queries}
\end{figure}

%% file: relatedWork.tex
\section{Related Work}
Research on improving the reliability of large models has produced several methodological families. We briefly review the most relevant 
lines of work.However, these methods address related but different objectives, and do not directly target document-level question answering with completeness and faithfulness .
\paragraph{Prompt-based structured reasoning: Chain-of-Thought (CoT), Tree-of-Thought (ToT), and their variants} When questions of coverage or completeness are raised in the context of LLMs, prior work most often points to CoT and ToT style reasoning. However, these methods address a fundamentally different problem setting.
CoT prompting~\cite{wei2022cot} and ToT reasoning~\cite{yao2023tot} encourage the model 
to generate intermediate steps or explore multiple reasoning branches. 
Self-consistency decoding~\cite{wang2023selfconsistency} further samples 
multiple reasoning paths to vote on the final answer.  
They primarily address tasks where the objective is to derive a solution from a compact problem statement. In these settings, reasoning proceeds by iteratively refining model-internal thought states.
In contrast, our setting requires exhaustive extraction and verification of information from a single source document. TOT does not explicitly structure reasoning around document-grounded answer searching. \textit{In their structures, once an initial interpretation is formed, subsequent search operates largely within the model’s internal reasoning space rather than repeatedly re-grounding in the source text.} As a result, omissions introduced during early interpretation cannot be systematically recovered, and document-level correctness cannot be guaranteed.
\paragraph{Retrieval-augmented generation (RAG)}
RAG reduces hallucinations by grounding model outputs in externally retrieved evidence~\cite{lewis2020rag}. Given a query, RAG retrieves the top-$k$ semantically relevant text chunks and conditions generation on these retrieved contexts, substantially improving factual alignment.
However, RAG fundamentally optimizes for relevance rather than coverage. Relevant information may lie outside the retrieved top-$k$. Retrieval is inherently biased by the initial query formulation, making some document elements systematically easier to retrieve than others. Increasing $k$ to improve recall quickly undermines the efficiency and purpose of retrieval. Multi-stage RAG or iterative RAG~\cite{fang-etal-2025-kirag} pipelines further amplify this issue: These new RAG variants first issue a broad query to retrieve relevant text chunks, and then iteratively refine the solution through more fine-grained retrieval and reasoning within the initially retrieved context.\textit{ However, information missed during the early retrieval stage cannot be recovered in later, more refined queries.
Moreover, RAG does not provide a validation mechanism to validate whether extracted content is even truly relevant or presented in the source document.} As a result, while RAG improves factual grounding, it offers no guarantees of document-level completeness or logical correctness.
\paragraph{Post-hoc verification and critic models.}
Post-hoc verification first lets the model generate a complete answer, and only afterwards checks whether it is logically or factually correct. Several works use self-evaluation, auxiliary critics, or 
LLM-as-a-judge mechanisms~\cite{selfcheckgpt2023} to detect 
hallucinations after the model has already produced an answer.  
\textit{Such methods can reject wrong outputs, but they cannot improve the model’s ability to generate correct ones. Moreover, they only assess what has been generated and therefore cannot detect or fix omissions.}So, this generation process often involves multiple trials and may not be able to finally find the answer.
In contrast, our method both exponentially improves recall and  performs validation during the generation process. After each result is validated by the model and wrong elements are pruned, the results are enumerated for final answer generation. Therefore, the content used for further generation has already been checked, preventing errors from propagating and enabling the model to produce correct solutions in a single pass. 
\paragraph{Program-aided and tool-using reasoning.}
Program-aided approaches, such as \textsc{Pal}~\cite{cheng2023pal}, \textsc{ReAct}~\cite{yao2023react}, and \textsc{Reflexion}~\cite{shinn2023reflexion}, improve performance by delegating parts of the reasoning process to executable programs. This operation ensures \emph{execution correctness}: the generated code can be run, debugged, and reliably perform arithmetic or external actions, leading to substantial gains on algorithmic tasks~\cite{cheng2023pal}. However, they do not prevent a critical failure mode that arises in high-rigor, long-form generation: when the required answer exceeds the model's single-pass generation capacity, \emph{omissions} and \emph{ordering errors} become unavoidable, even if every step is correctly executed.  Our work addresses a this fundamental bottleneck: 
instead of indiscriminately increasing generation length, we first generate a logical skeleton composed of structured sub-units, each of which serves as a prompt for a subsequent generation step. This induces a tree-structured, hierarchical expansion, where each sub-unit can be recursively refined and extended. Through this design, generation proceeds by progressively expanding the skeleton rather than by a single long pass, enabling structured outputs to scale to arbitrary depth while preserving ordering and coverage. While prior research proposed decomposition methods (e.g., multi-agent collaboration~\cite{khot2022decomposed}); they aim to leverage complementary expertise across models.

\paragraph{Decomposition frameworks and agent-style methods.}
These methods give the model substantial freedom to decide \emph{how} to act: when to search, when to call external tools, when to reflect, and when to terminate, such as ReAct~\cite{yao2022react}, Reflexion~\cite{shinn2023reflexion}, Voyager-style continual agents, and various LLM-based planners. This flexibility often improves task performance, as the model can opportunistically explore and select useful operations.  However, the design objective of these systems is
to \emph{increase} the model's freedom space, and unavoidably enlarged the error space. 
We keep the answer space fixed but repeatedly “squeeze” the error space through independent retrieval and validation steps, so that the probability of error shrinks by design exponentially, rather than by tuning.

%% file: theory1.tex
\section{Theoretical Analysis}
\subsection{EVE algorithm explanation}

\input{algorithm}

\subsection{Exponential Error Reduction through Aggregation and Multi-Stage Voting}
\label{Exponential Error Reduction through Aggregation and Multi-Stage Voting}
\begin{claim}[Exponential Omission Reduction of EVE]
Under the EVE architecture, where extraction and validation are performed through multiple independent queries at each stage, the overall probability of omission decreases exponentially in the extraction stage with the number of independent queries.
\end{claim}
\begin{proof}
When the system performs $M_e$ independent discovery extractions, let $p_{i,e}$ denote the probability that the attempt $i$ successfully extracts the entities. 
The probability that $M_e$ attempts fail to discover entities is
$ (1 - p_{i,e}).$Therefore, the probability that $e_i$ is discovered by at least one attempt is
$P_E=1 - \prod_{m=1}^{M_e} (1 - p_{i,e})$.
$P_E$ is the total extraction success rate. As $M_e$ increases, the omission  
decreases exponentially.
\end{proof}
\begin{table}[ht]
\centering
\caption{Theoretical comparison between free-form generation and EVE using the same underlying model accuracy parameters.}
\label{tab:comparison}
\begin{tabular}{lcc}
\toprule
\textbf{Metric} & \textbf{Free-form (baseline)} & \textbf{EVE} \\
\midrule
Discovery & 
$P_D^K = 0.6$ & 
$\begin{aligned}
&\prod_{i=1}^{K}\left( 1 - \prod_{m=1}^{M_i} ( p_{i,m}) \right) \\
&= 1 - 0.6^4 = 0.974
\end{aligned}$ \\
\addlinespace

Error Accept & 
$(1 - \varepsilon_{i,v}) = 0.8$ & 
$\begin{aligned}
&\binom{4}{3}0.8^3(0.2) + \binom{4}{4}0.8^4 \\
&+ \frac{1}{2}\binom{4}{2}0.8^2(0.2)^2 \\
&= 0.896
\end{aligned}$ \\
\addlinespace

Explored & 
fixed $M(5-12)/64 \approx 0.188$ & 
$(B \cdot F)^L = 1$ (ideal) \\
& & (e.g., 64) \\
\bottomrule
\end{tabular}
\end{table}
However, the extraction stage is designed for maximal coverage. As a result, the extracted set may contain both (i) multiple names that refer to the same underlying entity, and (ii) items that do not belong to this answer at all. Therefore, the validation stage is necessary: it removes elements that should not be present, and merges aliasing of the same answers. 
\begin{claim} The validation stage, with majority voting over independent validation queries, yields an exponentially decreasing error probability with respect to the number of independent queries.  
\end{claim}
\begin{proof}
Discovered candidates may include false positives. For each candidate, we perform $m_v$ independent validation queries (e.g., "Is this a valid answer? should this answer be merged with another answer? "). Intuitively, let $p_{i,M}$ denote the probability that the query $i$ correctly validates the entities. The probability of correctly validating all the extracted entities from the validation stage is:

\noindent$
\sum_{i=1}^{\frac{mv}{2}}
\binom{\frac{mv}{2}}{i}
\, p_{i,M}^{\frac{mv}{2}+i}
\left(1-p_{i,M}\right)^{\frac{mv}{2}-i}
$.
The error probability of majority voting can be bounded using standard
Chernoff/Hoeffding concentration inequalities.  In particular, it is a
classical result in ~\cite{mitzenmacher2005probability}
or ~\cite{boucheron2013concentration}
that if each query has an error rate at most $1-p_{i,M}<50\%$ and the
final decision is taken by majority vote over $m_v$ independent queries,
then
\begin{equation}
    P_{\mathrm{verify\text{-}error}}
    \;\le\;
    2 \exp\!\left(
        - \frac{m_v}{2} (1 - 2(1-p_{i,M}))^2
    \right).
\end{equation}
\end{proof}
Since the validation task is binary, random guessing attains an error rate of \(50\%\).
Thus, assuming \(1 - p_{i,M} < 50\%\) constitutes only a minimal sanity check for any reasonably capable LLM.
Combining these stages, results in an overall exponential reduction in the probability of false-positives or hallucinations with multiple independent queries, establishing the claim.

%% file: algorithm.tex
\begin{algorithm}[ht]
\caption{EVE: Structured Reasoning with Exponential Error Reduction}
\label{alg:eve}
\textbf{Input:} \\
\Indp 
Document $D$, LLM $M_e$, $M_v$, \\
Extraction query set $Q_{\mathrm{ext}} = \{q_{\mathrm{ext}}^{(1)}, \dots, q_{\mathrm{ext}}^{(M_e)}\}$, \\
Verification query set $Q_{\mathrm{val}} = \{Q_{\mathrm{val}}^{(1)}, \dots, Q_{\mathrm{val}}^{(M_v)}\}$ \\
Explanation query $Q_{\mathrm{f}}$, \\
\textbf{Output:} Structured report $R$
\Indm 

$C \gets \emptyset$ \tcp*{Candidate entities from all extractions}
\For{$i = 1$ \KwTo $M_e$}{
    $E \gets M\left(q_{\mathrm{ext}}^{(i)}, D\right)$ \tcp*{Independent extraction query $i$}
    $C \gets C \cup E$
}

$V \gets \emptyset$ \tcp*{Verified entities after voting}
\ForEach{$e \in C$}{
    $votes_e \gets 0$ \\
    \For{$j = 1$ \KwTo $M_v$}{
        $a \gets M\left(Q_{\mathrm{val}}^{(j)}, D, e\right)$ \tcp*{Independent verification query $j$}
        \If{$a = \text{True}$}{
            $votes_e \gets votes_e + 1$
        }
    }
    \If{$votes_e \ge M_v/2$}{
        $V \gets V \cup \{e\}$
    }
}

$R \gets \emptyset$ \\
\ForEach{$e \in V$}{
    $d \gets Q_f(\text{Describe}(e), D)$ \tcp*{Describe each verified entity}
    $R \gets R \cup \{(e, d)\}$
}

\Return $\text{Format}(R)$
\end{algorithm}

As shown in Algorithm \ref{alg:eve}, at a high level, content-based question answering can be viewed as a multi-stage process: the system first searches and generates a set of candidate elements, then validates them against the source content, and finally enumerates the validated elements into a complete answer.
When the required output exceeds the effective generation capacity of a single LLM pass, the explanation must be progressively extended from a high-level reasoning skeleton to more detailed, lower-level descriptions.
Precisely, EVE can be decomposed into 3 stages:
\begin{enumerate}
    \item \textbf{Exhaust Stage}: $m_e$ independent searches to identify candidate answer elements.   
    \item \textbf{Validation Stage}: $m_v$ independent validation queries per candidate, with voting to validate the candidate elements.
    \item \textbf{Enumeration Stage}: an upper-level LLM generation constructs a structured intermediate representation
    that encodes all information required for downstream expansion.
    This representation is designed to be algorithmically readable and
    drives subsequent iterative procedures. The lower level LLM generation is prompted with both the original
document and the encoded structure to generate detailed, element-level content.
\end{enumerate}
Finally, all the information is stacked together to generate a complete report.

Algorithm~\ref{alg:eve} describes the EVE framework for structured reasoning under a
closed-world assumption. The input document $D$ is assumed to contain all relevant
information.
In the extraction stage, the query set $Q_{\mathrm{ext}}$ consists of $M_e$ independent
and parallel extraction queries, each designed from a different perspective to extract
relevant candidate entities from the document. The union of all extracted results forms
the candidate set $C$, which intentionally prioritizes recall and may include spurious
entities as well as multiple forms referring to the same underlying entity.
In the validation stage, each candidate $e \in C$ is evaluated using $M_v$ independent
validation queries from $Q_{\mathrm{val}}$. The responses are aggregated by majority
voting, and only candidates receiving at least $M_v/2$ positive votes are retained.
Beyond filtering incorrect candidates, this stage also implicitly consolidates entities:
candidates that are judged to refer to an already validated entity are treated as aliases
and grouped under the same canonical entity, rather than as separate
entries.

Finally, for each validated entity, an explanation query $Q_f$ is used to generate a
document-grounded description conditioned on both the original document and the validated
entity. These descriptions are generated independently and assembled into the final
structured report $R$. By separating exhaustive extraction, voting-based validation,
and entity-wise explanation, EVE replaces a single high-variance generation with a
sequence of low-variance, verifiable steps, enabling exponential error reduction.

%% file: thoery2.tex
\subsection{Robust Enumeration via Multiple Attempts}
Crucially, the failure of single-pass LLM generation in enumeration tasks is not solely due to insufficient generation length.
Even if the length budget were unlimited, free-form reasoning would still be prone to omission.
Enumerating all relevant elements is inherently error-prone without a guidance of ordering.
In the absence of a predefined skeleton, the model has no principled criterion to determine whether the enumeration is complete, leading to missed elements and replicates.

Generation length further exacerbates this issue.
When the number of elements grows, the model may truncate its output before completing the enumeration, even if additional cases are implicitly recognized.
As a result, length constraints and correctness failures interact, making single-pass generation unreliable for safety-critical settings.

By contrast, the skeleton-based design in EVE explicitly separates correctness from capacity.
The upper-level controller enforces a structured enumeration process, while the lower-level LLM is responsible only for local reasoning within each attempt.
Through repeated, independent generation attempts governed by an explicit iterative algorithm, EVE ensures that all elements are systematically covered.

\subsection{Overall Completeness Probability}
Table~\ref{tab:comparison} is intended as an illustrative example to offer intuition about the magnitude of improvement that can be expected of EVE compared to typical LLM performance levels. To fairly demonstrate,
 we adopt a typical and homogeneous parameter setting, 
\(M_e = 4\), \(M_v = 4\), \(p_{i,e} =0.6\), and \(p_{i,M} = 0.7\).
Specifically, we assume four independent extraction attempts and four validation attempts, with constant success probabilities for each attempt.
In practice, the probability of obtaining a correct result varies across different prompts, tasks, and model capabilities.
However, to isolate the algorithmic effect and provide a clear illustration, we intentionally abstract away these variations and consider fixed representative values.
The selected probabilities are consistent with ranges reported in prior literature, and the detailed justification of these parameter choices is provided in Appendix~\ref{tab:param_justification}.

%% file: STPA.tex
\section{Experiment: STPA as a High-Rigor Reasoning Benchmark}
To evaluate the reliability and rigor of the EVE framework, we apply it to automated System-Theoretic Process Analysis (STPA), a canonical instance that requires high-rigor and has closed-world reasoning tasks. STPA is a safety-critical task whose objective is to systematically identify all Unsafe Control Actions (UCAs) that may place the system into hazardous states. Unlike open-ended reasoning tasks, STPA requires strict completeness and precision: missing a single UCA or its causal propagation path may leave critical safety risks unaccounted for. This makes STPA an ideal benchmark for evaluating whether LLMs can be constrained to produce reliable, non-hallucinatory, and complete structured outputs. STPA has seen ubiquitous research and use in the system safety literature and although out of scope of this paper, fuller descriptions can be found elsewhere\cite{leveson2011engineering}

\begin{figure*}[ht]
    \centering
    
    \begin{subfigure}[b]{0.32\textwidth}
        \centering
        \includegraphics[trim=0 0 0 00, clip,width=\linewidth]{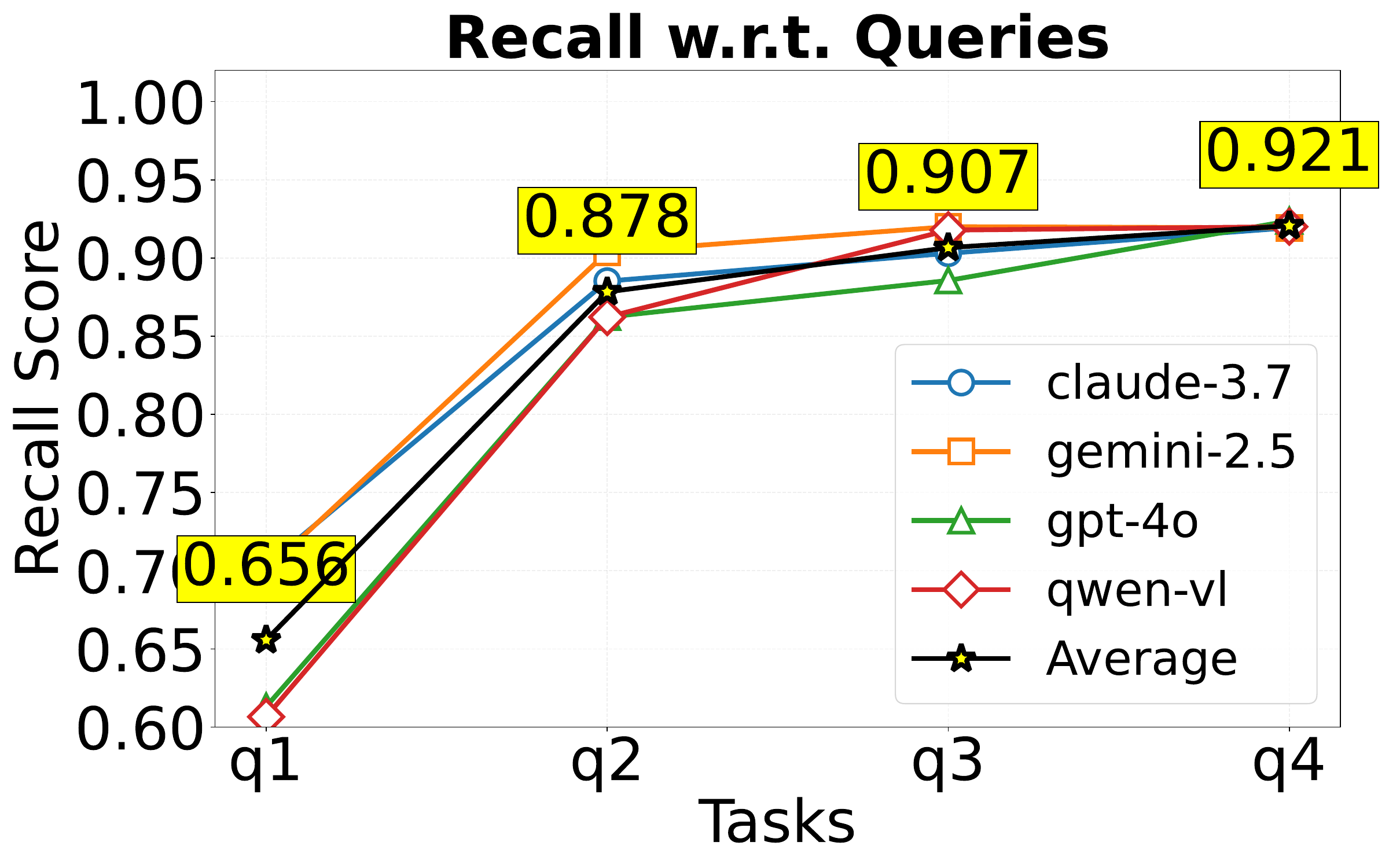}
        \caption{Recall Performance across Extraction Numbers}
        \label{fig:recall_performance}
    \end{subfigure}
    \hfill
    \begin{subfigure}[b]{0.32\textwidth}
        \centering
        \includegraphics[trim=0 0 0 00, clip,width=\linewidth]{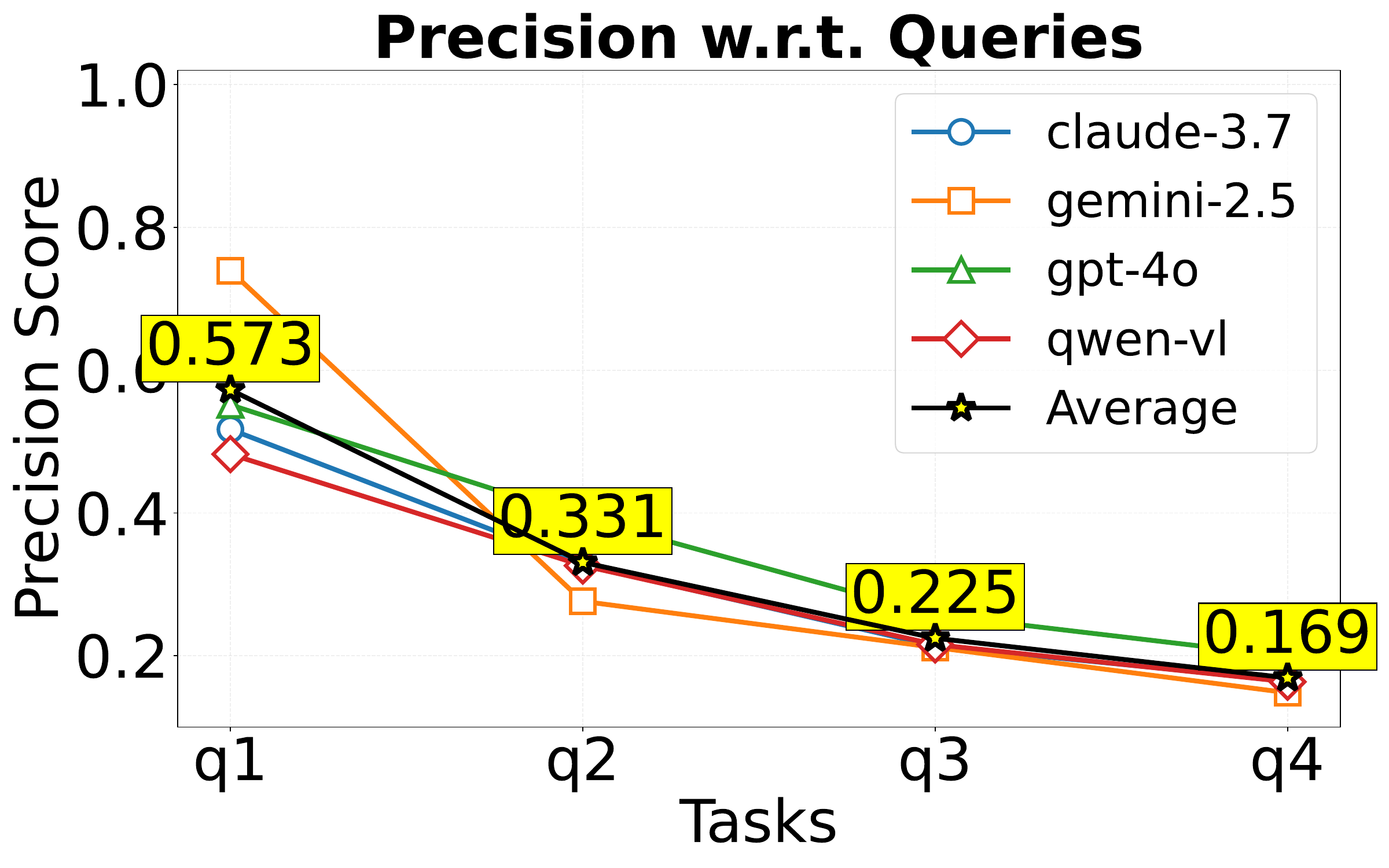}
        \caption{Precision Performance across Extraction Numbers}
        \label{fig:precision_performance}
    \end{subfigure}
    \hfill
    \begin{subfigure}[b]{0.32\textwidth}
        \centering
        \includegraphics[trim=0 0 0 00, clip,width=\linewidth]{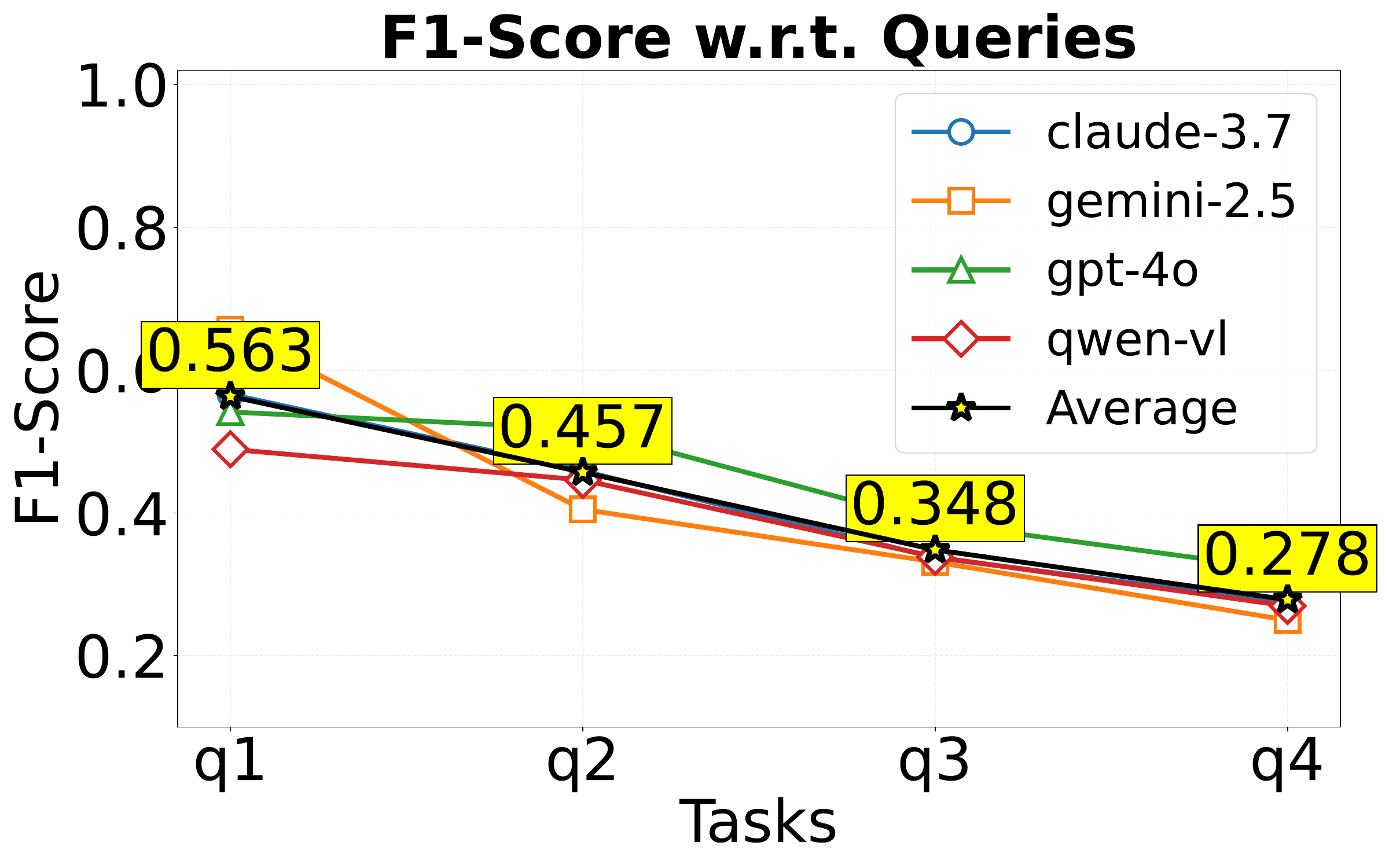}
        \caption{F1-Score Performance across Extraction Numbers}
        \label{fig:f1_performance}
    \end{subfigure}
    
    \caption{Extraction performance with varying numbers of independent extraction queries (q1 to q4), without validation}
    \label{Extraction performance with varying numbers of independent extraction queries (q1 to q4), without validation}
\end{figure*}
    
    
\begin{figure*}[ht]
    \centering
    
    \begin{subfigure}[b]{0.32\textwidth}
        \centering
        \includegraphics[trim=0 0 0 00, clip,width=\linewidth]{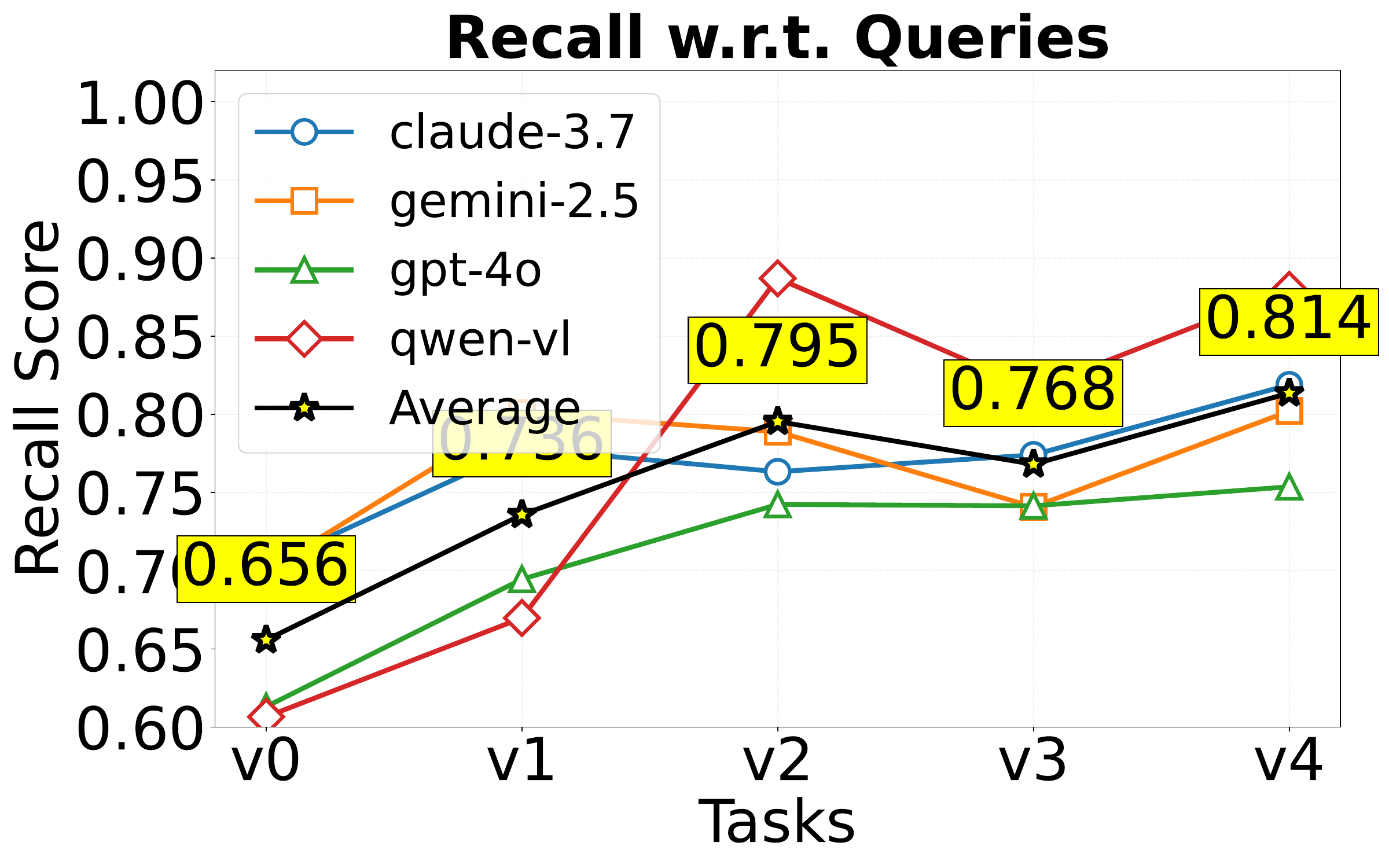}
        \caption{Recall Performance across Extraction Numbers}
        \label{fig:recall_performance2}
    \end{subfigure}
    \hfill
    \begin{subfigure}[b]{0.32\textwidth}
        \centering
        \includegraphics[trim=0 0 0 00, clip,width=\linewidth]{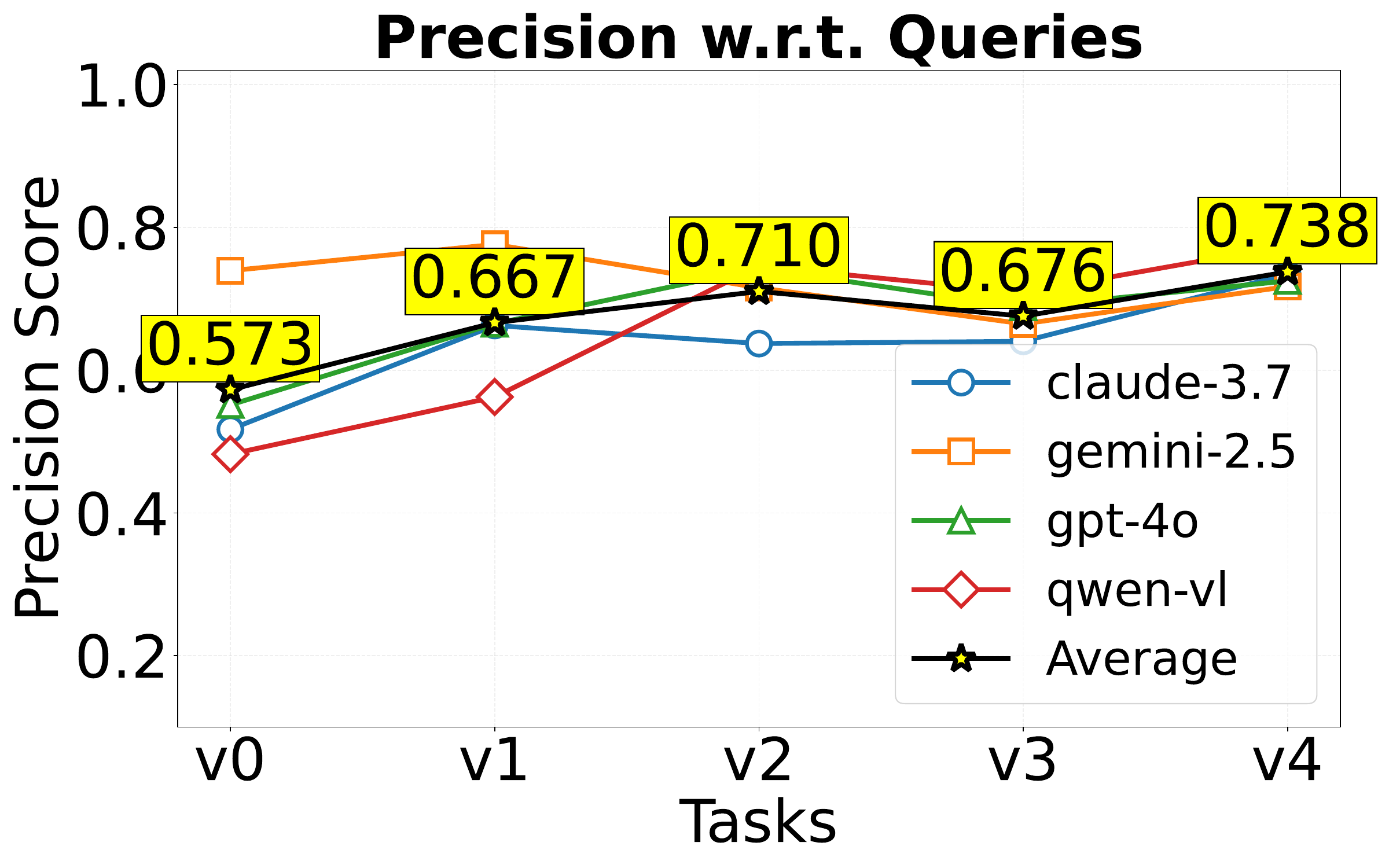}
        \caption{Precision Performance across Extraction Numbers}
        \label{fig:precision_performance2}
    \end{subfigure}
    \hfill
    \begin{subfigure}[b]{0.32\textwidth}
        \centering
        \includegraphics[trim=0 0 0 00, clip,width=\linewidth]{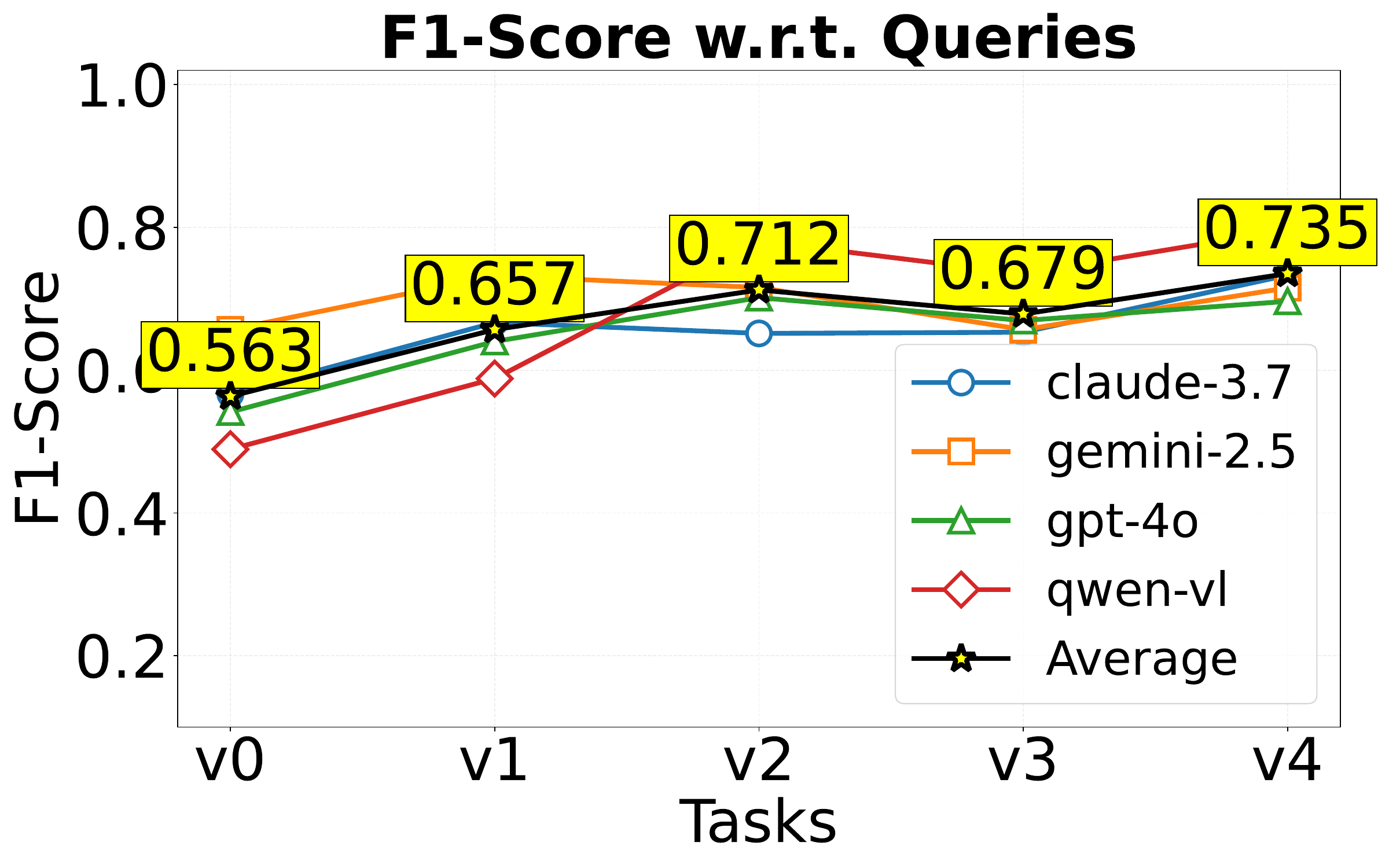}
        \caption{F1-Score Performance across Extraction Numbers}
        \label{fig:f1_performance2}
    \end{subfigure}
    
    \caption{Effect of validation intensity (v0 to v4 independent queries) on precision and recall after extraction}
    \label{Effect of validation intensity (v0 to v4 independent queries) on precision and recall after extraction}
\end{figure*}
The detail of STPA rule is specified in Appendix \ref{STPA reasoning and enumeration process}. 
From a computational perspective, given a system description, the analysis proceeds in 2 main steps.
1) to identify all control-related components in the system, such as sensors, controllers, and actuators. We study how different extraction strategies affect the correctness of this component discovery step. 2) 
Using a classical combinatorial procedure, the identified components as well as all their possible unsafe actions are permuted to illustrate every safety concerns. This enumeration can be as simple as a nested for loop.

Notably, traditional STPA analysis is mainly done manually, and no structured dataset of STPA results exists in the literature. This paper offers not only a formalized dataset based on 100 existing patents that describer various technical or engineered systems. This paper also provides automated, rigorous LLM-based STPA methods. In other words, we propose the \textit{ first method for systematically analyzing and automating STPA, and the first STPA dataset. The dataset and the code for this paper are also in our code repository}.


\begin{table}[t]
\centering
\caption{Comparative Average Performance Summary}
\label{tab:avg_performance_summary}
\begin{tabular}{llccc}
\toprule
\textbf{Task} & \textbf{Stage} & \textbf{Recall} & \textbf{Precision} & \textbf{F1\_Score} \\
\midrule
\multirow{4}{*}{Extraction}
 & q1 & 0.6557 & 0.5725 & 0.5634 \\
 & q2 & 0.8783 & 0.3308 & 0.4569 \\
 & q3 & 0.9066 & 0.2246 & 0.3484 \\
 & q4 & 0.9206 & 0.1690 & 0.2784 \\
\midrule
\multirow{5}{*}{Validation}
 & v0 & 0.6557 & 0.5725 & 0.5634 \\
 & v1 & 0.7358 & 0.6666 & 0.6568 \\
 & v2 & 0.7954 & 0.7103 & 0.7124 \\
 & v3 & 0.7681 & 0.6756 & 0.6786 \\
 & v4 & 0.8136 & 0.7377 & 0.7353 \\
\bottomrule
\end{tabular}
\end{table}

\paragraph{Experimental Protocol.}
In the extraction phase, we query the model separately using 1 to 4 independent prompts.
Each prompt is self-contained and asks the model to produce the complete set of target answers from the document—they are not designed to extract partial information incrementally.
For example, in our STPA task, every prompt explicitly requests the model to list all sensors (or all actuators) mentioned in the specification, rather than querying sub‑categories.
The responses from the different prompts are then merged via union to form the extracted candidate set for future validations.
We evaluate the extractions obtained with 1 to 4 prompts against the ground‑truth answers.
For each query budget, the specific prompts are randomly sampled from a predefined pool of 4 questions, and results are averaged over these samples. This ensures that performance differences reflect the number of independent extraction attempts rather than the particular wording or ordering of individual prompts.
Table \ref{tab:avg_performance_summary} summarizes the results across different numbers of extraction prompts; detailed per‑model performance is provided in Table \ref{Performance Comparison of 1 vs Multiple Extraction for Different Models} in Appendix \ref{Detailed Performance Tables for Extraction and Validation Stages}. (The full set of prompts is listed in Appendix \ref{Extraction and Verification Stage Query Design}.) 

\paragraph{Impact of Multi-Query Extraction.}
As shown in Figure \ref{Extraction performance with varying numbers of independent extraction queries (q1 to q4), without validation}, we observe a consistent and monotonic increase from 65.6\% to 92.1\% in recall as the number of extraction queries grows from 1 to 4.
This trend shows independent extraction queries reduce stochastic omission errors by covering complementary perspectives of the document. Notably, the recall improvement follows a characteristic \emph{rapid-rise-then-saturation} pattern.
Introducing a second extraction query yields a sharp recall increase, while additional queries beyond the third contribute only marginal gains. 
Across models, all recall approaches a stable regime around $92\%$. From a cost–performance perspective, introducing a second extraction query already captures the majority of recoverable entities, yielding the largest marginal gain in recall. While additional queries further improve recall, their benefits diminish rapidly, suggesting that two independent queries represent a cost-effective operating point, whereas three or more queries are primarily justified when near-saturation recall is required.

In contrast, precision and F1 decreases steadily as more extraction queries are added.
This behavior is expected and intentional: recall-oriented extraction prioritizes coverage and admits redundant or spurious candidates.

\paragraph{Impact of Multi-Query Validation}
 Figure \ref{Effect of validation intensity (v0 to v4 independent queries) on precision and recall after extraction}, the number of extraction queries is fixed to 4 and the number of validation queries changes from 0 to 4. There is a detailed Table \ref{Performance Comparison of 1 vs Multiple validations for Different Models_a2} is shown in Appendix \ref{Detailed Performance Tables for Extraction and Validation Stages}. 
The prompts of the validation queries are also randomly selected from a pool of 4 validation questions, so the performance change shows the impact of different number of validations not the wording of the particular prompts.
A single validation could incorrectly reject true positives, thereby forfeiting the high recall obtained from the multi-query extraction. 
To mitigate this risk, we employ a majority-voting scheme over $m_v$ independent validation queries per candidate. 
This ensemble approach reduces the probability of erroneous rejection, thereby safeguarding recall while simultaneously boosting precision through the removal of false positives.

As shown in Figure \ref{Effect of validation intensity (v0 to v4 independent queries) on precision and recall after extraction}, relying solely on extraction without validation (v0) results in uniformly low performance across recall, precision, and F1-score. Introducing even a single validator leads to a substantial improvement in all three metrics, while adding a second validator further boosts performance. Beyond two validators, the gains become marginal and exhibit slight fluctuations, indicating diminishing returns. These results highlight the necessity of multi-query validation and suggest that re-validating extracted candidates against the source document—even with a single additional query—can significantly enhance overall performance. \textit{Despite its effectiveness, such post-extraction validation is rarely adopted in existing LLM pipelines, underscoring the practical value of our proposed “back-check” validation strategy.}

"Second is the best" mirrors the observation in the extraction stage, where the largest all parameter increment is obtained by introducing a second extraction or validation query. 
Together, these results indicate that a \emph{sweet spot} exists at two independent queries—whether for extraction or validation—beyond which additional queries contribute little additional benefit. 
Thus, while the multi-query design is essential for reducing stochastic omission and hallucination errors, a relatively small ensemble size (two or three) is sufficient to capture the majority of the achievable performance uplift in both stages.

\paragraph{The Impact of Enumeration.}
Once all relevant components are reliably identified, the remaining challenge lies in systematically enumerating all possible unsafe control action (UCA), a complete ordered permutations over all components and all the safety concerns.
Rather than relying on free-form generation, EVE adopts an explicit permutation-based enumeration strategy, implemented as a nested-loop procedure over components and possible unsafe types.
This design choice is essential, as the combinatorial space quickly exceeds the effective generation capacity of LLMs.

Our experiments further confirm the necessity of enumeration as a structural backbone for long-form reasoning.
Without enumeration, even when prompted to exhaustively describe all possible safety concerns, LLMs tend to produce only a small subset of plausible cases.
In contrast, real-world control systems typically involve multiple sensors, controllers, and actuators, leading to hundreds of distinct UCA combinations that must be explicitly considered.The number of UCAs grows proportionally to
$N_S \times N_C \times N_A \times F$,
where $N_S$, $N_C$, and $N_A$ denote the numbers of sensors, controllers, and actuators, respectively,
and $F$ denotes the number of unsafe modes.

As shown in Table~\ref{tab:stpa_example_count_comparison}, introducing enumeration increases the number of covered safety scenarios by orders of magnitude, demonstrating that structured enumeration is indispensable for achieving completeness beyond the limits of single-pass generation.

\begin{table}[t]
\centering
\caption{STPA example count comparison with and without state machine.
No-State reports average examples per file; State-Machine reports average trajectories per file.}
\label{tab:stpa_example_count_comparison}
\setlength{\tabcolsep}{5pt}
\begin{tabular}{lcc}
\hline
\textbf{Model} & \textbf{No-Enumeration} & \textbf{Enumeration} \\
\hline
Claude-3.7 Sonnet & 4.03 & 87.4 \\
Gemini-2.5 Pro    & 4.23 & 77.8 \\
GPT-4o            & 3.77 & 81.4 \\
Qwen VL-Max       & 4.13 & 220.0 \\
\hline
\textbf{Avg.}     & \textbf{4.02} & \textbf{110.8} \\
\hline
\end{tabular}
\end{table}

Table~\ref{tab:stpa_example_count_comparison} reports the number of enumerated results produced by different models.
The variation across models primarily stems from differences in the numbers of components extracted in the upstream stages, which directly determine the size of the subsequent enumeration space.
In this table, all results are obtained under the strongest configuration, using four independent extraction queries followed by four independent validation queries.

Notably, the Qwen model yields a substantially larger number of enumerated cases than the other models.
As shown in Figure~3 and Figure~4, Qwen exhibits relatively weak performance under single-query settings, but benefits disproportionately from multi-query extraction and validation.
When combined with the proposed EVE framework, Qwen achieves higher recall, precision, and F1-score than the other models, resulting in a larger and more complete set for the enumerations.
This observation suggests that, beyond raw single-query accuracy, certain models are particularly well suited to structured, multi-query reasoning pipelines as shown in our EVE framework.
More broadly, it points to a promising direction for future work, where model–framework compatibility may play a critical role in maximizing reliability and correctness.

%% file: limitations.tex
\section{Limitations and Boundary Conditions}
\paragraph{The performance of EVE is bounded by inherent limits of language-based reasoning.}
It is easy to observe that, despite our best efforts to systematically address safety concerns, the performance of EVE saturates below perfect reliability.
In our experiments, recall peaks at approximately 92\%, while precision plateaus around 73.8\%.
Although these results represent substantial improvements over the single-query baseline—yielding up to 40\% relative gains in recall and up to 36\% gains in precision—the resulting performance still falls short of the strict requirements demanded by safety-critical applications.

This limitation reflects a more fundamental constraint.
LLMs operate over natural language, which is inherently ambiguous and underspecified: correctness is rarely binary, and semantic boundaries are often fuzzy rather than exact~\citep{bender2020climbing}.
As a consequence, even extensive multi-query extraction and validation cannot fully eliminate residual uncertainty.
This gap is manifestation of the intrinsic limits of language-based statistical modeling. Addressing this limitation likely requires complementary mechanisms beyond language modeling.
One promising direction is to integrate formal methods or symbolic reasoning tools that can externally verify safety constraints and rule out logically invalid cases.
Such approaches would require transforming language-derived outputs into explicit mathematical or logical representations, enabling rigorous checking that is orthogonal to probabilistic inference.
In this view, ensemble learning over LLM queries is a necessary but insufficient condition for safety, and must ultimately be coupled with external verification tools to achieve stronger guarantees.
In addition, EVE is intentionally designed to be model-agnostic and compatible with general-purpose LLMs.
While this design choice ensures broad applicability, it also constrains access to deep, tacit domain knowledge specific to control systems, which could potentially further improve the correctness.

\paragraph{The Underestimated Challenge of Entity Clustering}
A critical yet often under-discussed limitation lies not in the \emph{binary validation} of candidate correctness, but in the semantic clustering of aliases.
Binary validation is relatively straightforward. Queries such as “Is X a valid answer, based on the context?” are well-defined, document-grounded classification tasks. With the mainstream models, accuracy in this step can be high. (In practice, wrong answer is very rarely observed). However, clustering is inherently ambiguous and technically challenging. The true difficulty emerges when multiple extracted surface forms (e.g., “front radar,” “obstacle detection sensor,” “LiDAR unit”) refer to the same underlying entity in the document. Precision degradation mainly arises from the absence of a deterministic decision boundary for semantic equivalence, rather than incorrect acceptance of invalid elements.
Precisely, Dictionary-based methods fails due to lack of context; nearly any two distinct strings are trivially considered non‑synonymous. Using an LLM as a judge for synonymy is overly conservative and unreliable. When asked in isolation “Are A and B the same thing?,” state‑of‑the‑art models tend to err toward “no,” rarely affirming synonymy unless the phrases are nearly identical, leading to excessive duplication. The only robust method is contextual grounding: The reliable approach—and the one we adopt—is to ask the model, \emph{within the specific document context}, whether two mentions refer to the same component. This shifts the problem from abstract synonymy to concrete co‑reference, yielding significantly more accurate clustering. Nevertheless, this itself requires careful prompt design and is not foolproof.
\section{Conclusion}

In summary, we present EVE, a structured framework for document-grounded reasoning that improves both faithfulness and completeness by explicitly decomposing answer construction into three stages:
(1) Exhaustive extraction, where multiple independent queries are used to maximize recall of candidate elements;
(2) validation, where majority voting over multiple binary queries filters out false positives while preserving recall; and
(3) enumeration, where a systematically generated reasoning skeleton guides the language model to produce a complete and ordered answer.

Unlike free-form prompting, EVE constrains the generation process to a verifiable, skeleton-driven pipeline, delegating fine-grained content generation only after the structural reasoning is complete.
Empirically, this design enables consistent and simultaneous improvements in recall, precision, and F1-score: under the strongest configuration with four extraction and four validation queries, recall and precision increase by up to 24\% and 29\%, respectively, with a corresponding 31\% gain in F1-score. This effectively breaks the long-standing trade-off between coverage and accuracy typical of single-pass LLM generation. At the same time, we emphasize that EVE does not eliminate the fundamental limits of language-based reasoning.At the same time, we emphasize that EVE does not eliminate the fundamental limits of language-based reasoning.
Addressing these limitations—potentially through tighter integration with formal verification or domain-specific
reasoning tools—remains an important direction for future work.

%% file: Appendix.tex
\clearpage
\appendix

\section{Practical Guidelines for Designing Effective LLM Queries}\label{Practical Guidelines for Designing Effective LLM Queries}
The reliability of any LLM-based analysis, including structured frameworks, critically depends on the quality of the initial queries. Vague or poorly structured prompts lead to ambiguous, incomplete, or irrelevant outputs, undermining subsequent verification and reasoning steps. This appendix provides concrete, actionable guidelines with examples to craft precise and directive prompts.

Guideline 1: Use Structured Verb-Fronting
Begin each key instruction with a strong, imperative verb that clearly states the action required. Avoid nominalizations or passive phrasing.

Weak (Unfocused): "A discussion of analog input interfaces and their possible sensors."

Strong (Directive): "List all analog input (AI) interfaces. For each, identify the specific types of physical sensors that could be connected."

Why it works: Verbs like “List,” “Identify,” “Analyze,” “Compare,” or “Extract” set an unambiguous task, reducing the LLM’s discretion to “discuss” the topic in an open-ended manner.

Guideline 2: Embed Explicit Micro-Questions
Transform broad topic headings into specific, answerable sub-questions. This acts as a “reasoning anchor” for the LLM.

Weak (Topic List): "Consider: 1. Analog inputs. 2. Digital inputs. 3. Communication buses."

Strong (Question List): "For each interface category, answer: 1. What specific sensor models connect to analog inputs? 2. What state detection devices connect to digital inputs? 3. What data streams are transmitted via communication buses?"

Why it works: It forces the LLM to produce factual answers to each micro-question, ensuring comprehensive coverage of the topic and preventing abstract summarization.

Guideline 3: Employ Negative Clarification
Explicitly state what should not be included to narrow the scope and prevent common failure modes.

Weak (Ambiguous Scope): "List system components."

Strong (Bounded Scope): "List all physical hardware components (e.g., controllers, sensors, actuators). Exclude software modules, logical processes, and data flows."

Why it works: It pre-emptively counters the LLM’s tendency to over-generalize or include conceptually related but task-irrelevant entities, significantly improving precision.

Guideline 4: Enforce Output Format Constraints
Specify a strict output template (e.g., JSON, XML, a markdown table) with predefined fields. This structures the LLM’s thought process and enables automated parsing.

Weak (Free-form): "Tell me about the sensors."

Strong (Structured): "Output a JSON array. Each object must have these keys:
\begin{verbatim}
{
  "component_name": [...],
  "component function": [...],
}
\end{verbatim}
Why it works: It moves the task from “generate text about X” to “fill in a data structure for X,” which is a more constrained and less hallucination-prone operation for the LLM.

Guideline 5: Assign a Specific Analytical Role
Frame the task by assigning the LLM a professional or functional role that shapes its reasoning perspective.

Weak (Generic): "Analyze this control system for hazards."

Strong (Role-based): "You are a safety certification engineer conducting a formal hazard analysis. Apply the STPA methodology to analyze this control system. Your output must reflect regulatory-grade rigor."

Why it works: The role primes the model with a specific set of assumptions, priorities, and jargon, leading to more context-appropriate and rigorous outputs.

Summary: Applying these guidelines transforms prompts from vague invitations for discussion into precise engineering instructions. This directly increases the initial precision and actionability of the LLM’s output, which is then amplified by downstream structured reasoning and verification modules.

Guideline 6: Prefer Structural Prompting over Instruction Tuning

Simply instructing the model to be clear, detailed, or careful often has limited effect on output quality, with results remaining largely invariant to instruction phrasing. Such instruction-level variations rarely alter the effective decision space of the model and therefore provide limited control over completeness or consistency.

Weak (Instruction-based): Please carefully and clearly analyze the hazards in this control system.''

Strong (Structurally constrained): Fill in the following template for each identified hazard: $\langle$error type$\rangle$, impacting $\langle$component$\rangle$, resulting in $\langle$consequence$\rangle$, which is dangerous to $\langle$hazard$\rangle$.''

Why it works:
Structural prompting transforms the task from open-ended explanation into constrained decision completion. By requiring the model to explicitly fill predefined fields, the output space is restricted to valid and interpretable configurations. This improves output consistency and reduces ambiguity and omissions, particularly in high-rigor reasoning settings where structured justification is required.
\section{Extraction and Validation Stage Query Design}
\label{Extraction and Verification Stage Query Design}
\subsection{Sensor Identification}

Sensor extraction is performed in two stages:

\textbf{Recall Stage.} Five questions are used to maximize coverage:  
(1) direct extraction from text,  
(2) interface-based inference,  
(3) functional requirement deduction,  
(4) control-loop tracing, and  
(5) optional diagram-based reasoning.  

\textbf{Precision Stage.} Four verification questions refine the results by:  
(1) validating true physical sensors,  
(2) clustering equivalent names,  
(3) adversarial completeness checking, and  
(4) incremental consolidation across perspectives.

\subsection{Actuator Identification}

The same multi-perspective strategy is applied to actuators.  
Actuator recall focuses on identifying final physical components that directly influence external variables using:  
(1) STPA boundary reasoning,  
(2) energy-flow analysis,  
(3) signal-to-physical conversion tracing, and  
(4) functional deduction.  

Precision verification similarly validates real actuators, clusters aliases, removes non-actuators, and consolidates results.

\subsection{Voting and Consolidation}

To integrate the results from multiple perspectives, a lightweight consolidation process is applied based on four steps: 
(1) prioritizing results from earlier and more reliable questions, 
(2) using later questions mainly for cross-verification and completeness checking, 
(3) merging candidate entities that refer to the same physical component, and 
(4) assigning confidence according to the number of perspectives that support each entity. 
The final output aggregates all validated answers and produces a unified entity summary in a consistent structured format.
All questions return results using similar schema:

\begin{verbatim}
{
  "component_name": [...],
  "component function": [...],
}
\end{verbatim}
\section{Detailed Performance Tables for Extraction and Validation Stages}
\label{Detailed Performance Tables for Extraction and Validation Stages}
\begin{sidewaystable}[htbp]
\centering
\caption{Performance Comparison of 1 vs Multiple Extraction for Different Models}
\label{Performance Comparison of 1 vs Multiple Extraction for Different Models}
\begin{tabular}{|c|c|c|c|c|c|c|c|c|c|c|c|c|}
\hline
\multirow{2}{*}{Model} & \multicolumn{3}{|c|}{q1} & \multicolumn{3}{|c|}{q2} & \multicolumn{3}{|c|}{q3} & \multicolumn{3}{|c|}{q4} \\
\cline{2-13}
 & R & P & F & R & P & F & R & P & F & R & P & F \\
\hline
claude & 0.703 & 0.517 & 0.566 & 0.885 & 0.327 & 0.459 & 0.903 & 0.214 & 0.337 & 0.919 & 0.165 & 0.274 \\
\hline
gemini & 0.700 & 0.739 & 0.657 & 0.904 & 0.276 & 0.405 & 0.920 & 0.212 & 0.331 & 0.919 & 0.148 & 0.250 \\
\hline
gpt-4o & 0.613 & 0.551 & 0.541 & 0.862 & 0.393 & 0.518 & 0.886 & 0.257 & 0.386 & 0.924 & 0.199 & 0.320 \\
\hline
qwen & 0.607 & 0.482 & 0.489 & 0.862 & 0.326 & 0.446 & 0.918 & 0.216 & 0.338 & 0.920 & 0.164 & 0.270 \\
\hline
average & 0.656 & 0.573 & 0.563 & 0.878 & 0.331 & 0.457 & 0.907 & 0.225 & 0.348 & 0.921 & 0.169 & 0.278 \\
\hline
average(\%) & --- & --- & ---  & 34\% & -42\% & -19\% & 38\% & -61\% & -38\% & 40\% & -70\% & -51\% \\
\hline
\end{tabular}
\end{sidewaystable}

\begin{sidewaystable}[htbp]
\centering
\caption{Performance Comparison of 1 vs Multiple validations for Different Models}
\subcaption{Number of validation equals to the number of extraction}
\label{Performance Comparison of 1 vs Multiple validations for Different Models_a1}
\begin{tabular}{|
p{1cm}|
p{0.6cm}
p{0.6cm}
p{0.6cm}|
p{0.6cm}
p{0.6cm}
p{0.6cm}|
p{0.6cm}
p{0.6cm}
p{0.6cm}|
p{0.6cm}
p{0.6cm}
p{0.6cm}|
p{0.6cm}
p{0.6cm}
p{0.6cm}|
}
\hline
\multirow{2}{*}{Model} & \multicolumn{3}{|c|}{v0} & \multicolumn{3}{|c|}{v1} & \multicolumn{3}{|c|}{v2} & \multicolumn{3}{|c|}{v3} & \multicolumn{3}{|c|}{v4} \\
\cline{2-16}
 & R & P & F & R & P & F & R & P & F & R & P & F & R & P & F \\
\hline
claude
& 0.703 & 0.517 & 0.566 
& 0.768 & 0.593 & 0.630 
& 0.841 & 0.779 & 0.728 
& 0.877 & 0.729 & 0.743 
& 0.819 & 0.734 & 0.733 \\
\hline
gemini
& 0.700 & 0.739 & 0.657 
& 0.702 & 0.631 & 0.624 
& 0.796 & 0.817 & 0.755 
& 0.724 & 0.745 & 0.682 
& 0.802 & 0.717 & 0.715 \\
\hline
gpt-4o 
& 0.613 & 0.551 & 0.541 
& 0.712 & 0.671 & 0.651 
& 0.717 & 0.828 & 0.725 
& 0.763 & 0.846 & 0.746 
& 0.754 & 0.725 & 0.696 \\
\hline
qwen 
& 0.607 & 0.482 & 0.489 
& 0.770 & 0.623 & 0.663 
& 0.828 & 0.692 & 0.707 
& 0.855 & 0.741 & 0.742 
& 0.880 & 0.774 & 0.797 \\
\hline
average 
& 0.656 & 0.573 & 0.563 
& 0.738 & 0.630 & 0.642 
& 0.796 & 0.779 & 0.729 
& 0.805 & 0.765 & 0.728 
& 0.814 & 0.738 & 0.735 \\
\hline
(\%) 
& --- & --- & --- 
& 13\% & 10\% & 14\% 
& 21\% & 36\% & 29\% 
& 23\% & 34\% & 29\% 
& 24\% & 29\% & 31\% \\
\hline
\end{tabular}
\end{sidewaystable}
\begin{sidewaystable}[htbp]
\centering
\caption{Performance Comparison of 1 vs Multiple validation for Different Models}
\subcaption{The number of extraction=4}
\label{Performance Comparison of 1 vs Multiple validations for Different Models_a2}
\label{tab:1_vs_multiple_filtering}
\begin{tabular}{|
p{1cm}|
p{0.6cm}p{0.6cm}p{0.6cm}|
p{0.6cm}p{0.6cm}p{0.6cm}|
p{0.6cm}p{0.6cm}p{0.6cm}|
p{0.6cm}p{0.6cm}p{0.6cm}|
p{0.6cm}p{0.6cm}p{0.6cm}|
}
\hline
\multirow{2}{*}{Model}
& \multicolumn{3}{c|}{v0}
& \multicolumn{3}{c|}{v1}
& \multicolumn{3}{c|}{v2}
& \multicolumn{3}{c|}{v3}
& \multicolumn{3}{c|}{v4} \\
\cline{2-16}
& R & P & F
& R & P & F
& R & P & F
& R & P & F
& R & P & F \\
\hline
claude
& 0.703 & 0.517 & 0.566
& 0.778 & 0.663 & 0.667
& 0.763 & 0.637 & 0.652
& 0.774 & 0.640 & 0.653
& 0.819 & 0.734 & 0.733 \\
\hline
gemini
& 0.700 & 0.739 & 0.657
& 0.801 & 0.776 & 0.732
& 0.789 & 0.715 & 0.716
& 0.741 & 0.665 & 0.657
& 0.802 & 0.717 & 0.715 \\
\hline
gpt-4o
& 0.613 & 0.551 & 0.541
& 0.695 & 0.665 & 0.640
& 0.742 & 0.743 & 0.702
& 0.742 & 0.688 & 0.670
& 0.754 & 0.725 & 0.696 \\
\hline
qwen
& 0.607 & 0.482 & 0.489
& 0.670 & 0.562 & 0.588
& 0.887 & 0.746 & 0.781
& 0.816 & 0.709 & 0.735
& 0.880 & 0.774 & 0.797 \\
\hline
average
& 0.656 & 0.573 & 0.563
& 0.736 & 0.667 & 0.657
& 0.795 & 0.710 & 0.712
& 0.768 & 0.676 & 0.679
& 0.814 & 0.738 & 0.735 \\
\hline
(\%)
& --- & --- & ---
& 12\% & 16\% & 17\%
& 21\% & 24\% & 26\%
& 17\% & 18\% & 20\%
& 24\% & 29\% & 31\% \\
\hline
\end{tabular}
\end{sidewaystable}

\textbf{Model abbreviations.}

\emph{claude} denotes anthropic/claude-3.7-sonnet;

\emph{gemini} denotes google/gemini-2.5-pro;

\emph{gpt-4o} denotes openai/gpt-4o;

\emph{qwen} denotes qwen/qwen-vl-max.

\textbf{Metric abbreviations.}
R, P, and F denote Recall, Precision, and F1-score, respectively.
\input{Appendix2}

\section{Parameter selection table}
\label{tab:param_justification}
\begin{table*}[ht!]
\centering
\caption{Parameter values and their justifications.}

\begin{tabular}{p{0.1\textwidth} p{0.6\textwidth}}
\toprule
\textbf{Value} & \textbf{Literature Support \& Rationale} \\
\midrule
4 &
Represents a simple but realistic control unit: one controller, two sensors, one actuator.
A fundamental building block of larger systems. \\

0.7-0.8 &
Conservative estimate based on:
1) LLM zero-shot F1 scores on IE tasks ($\sim$70\%) with recall typically lower than precision \citep{he2023exploring};
2) Reported ``completeness'' deficits in evidence extraction \citep{lyu2024faithful};
3) Performance of advanced prompting methods like Chain-of-Thought on multi-step problems. \\

0.2 &
LLM-as-a-judge accuracy in verification/fact-checking tasks is often reported in the 80--85\% range
for non-adversarial cases \citep{manakul2023selfcheckgpt}.
We use an error rate of 0.2 (80\% accuracy). \\

0.85 &
Constrained, context-specific generation significantly improves accuracy.
Program-aided approaches show jumps to $>$90\% accuracy on well-defined sub-tasks \citep{cheng2023pal}.
We use a slightly conservative value. \\
\bottomrule
\end{tabular}
\end{table*}

%% file: Appendix2.tex
\section{STPA reasoning and enumeration process}
\label{STPA reasoning and enumeration process}
\begin{figure}[ht]
    \centering
    \includegraphics[width=\linewidth]{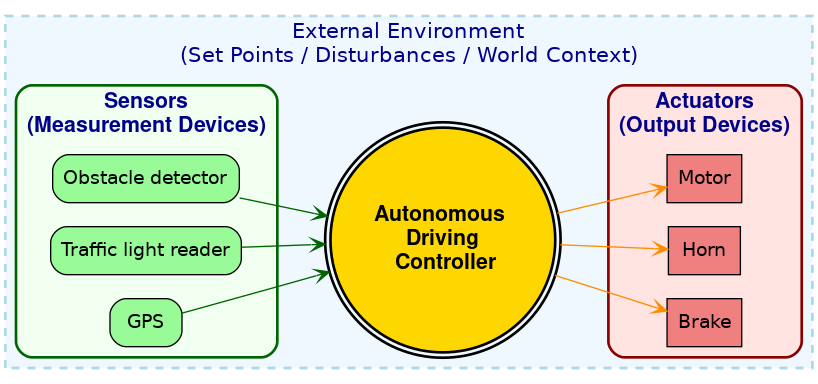}
    \caption{Illustration of a control diagram of a self-driving car}
    \label{control_system}
    \includegraphics[width=\linewidth]{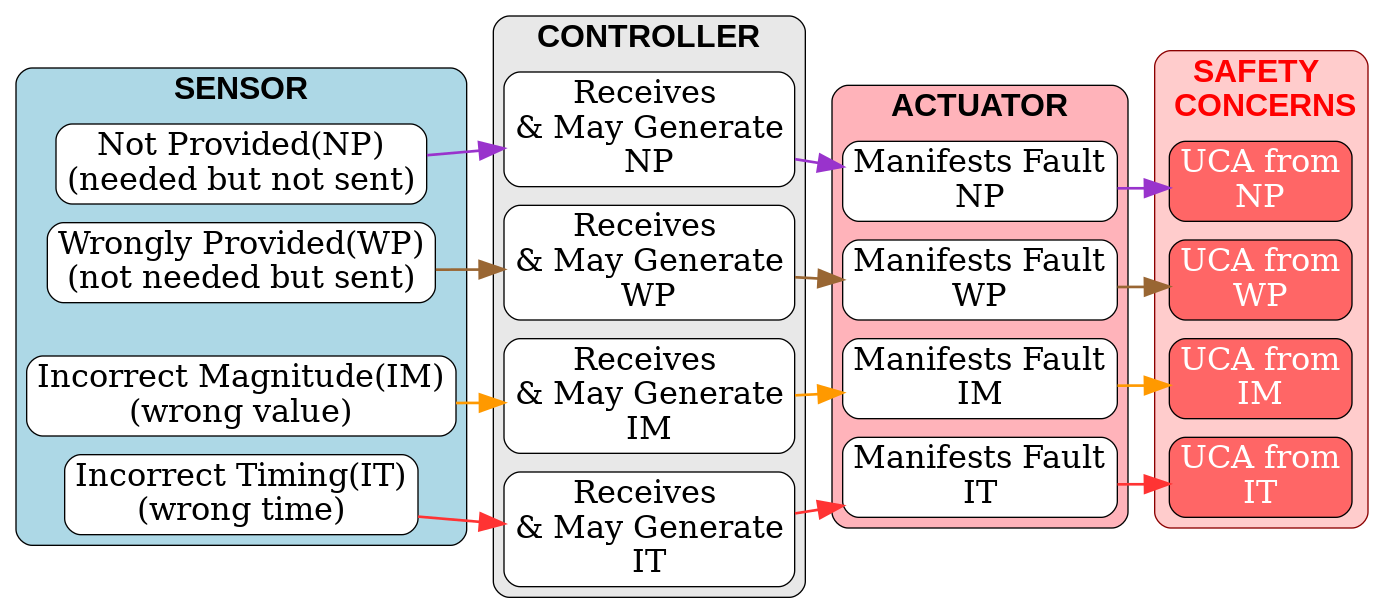}
    \caption{STPA Unsafe Control Action Propagation Diagram. The UCAs can be originated from close to out put nodes, such as actuator. The UCAs can also be originated from earlier nodes, controllers or sensors. The UCAs can only transmit within the same category.}
    \label{Propagation}
\end{figure}

STPA reasoning is based on component-wise state transitions within the system. The analysis focuses solely on externally observable states, rather than internal implementation details. As exemplified Figure \ref{control_system}, each system is modeled as a control structure composed of sensors, controllers, and actuators element-wisely. Each control component can lead to four types of safety concerns or UCA:
Not Provided When Needed (NP),
Wrongly Provided When Not Needed (WP),
Incorrect Magnitude (IM), and
Incorrect Timing (IT). Under this control-theoretic view, an UCA occurring at the state of a component propagates to downstream components along directed control dependencies, as illustrated in Fig.~\ref{Propagation}.
We assume that the \emph{type} of a UCA remains unchanged during propagation; for example, an NP-type UCA will not transform into WP, IM, or IT along the propagation path.

For tractability and clarity of reasoning, STPA further restrict UCA transmission to a single, non-branching control path.
That is, although a UCA may in reality influence multiple downstream components simultaneously, STPA reasoning considers only one propagation target at a time.
Each propagation instance is analyzed independently, such as from a sensor to a controller and then to a single actuator.
In summary, the total number of UCAs in a control system can be expressed as:$N_{\mathrm{UCA}} = F \bigl( N_A + N_C N_A + N_S N_C N_A \bigr)$
\noindent
where $F$ denotes the number of UCA per component defined as 4, and
$N_S$, $N_C$, and $N_A$ represent the numbers of sensors,
controllers, and actuators, respectively.
For example, in the system shown in Fig.~\ref{control_system},
$N_{\mathrm{UCA}} = 60$ when $N_S = 3$, $N_C = 1$, and $N_A = 3$.

Based on this structured propagation model, the role of the large language model is explicitly decomposed.
In the first step, the LLM identifies all control components in the system—namely sensors, controllers, and actuators—together with their relevant functional descriptions, using the proposed extraction and validation steps.

During enumeration, the framework systematically generates all possible UCAs by the aforementioned enumeration method: each control component has 4 UCA types multiplied by all the number of component permutations after it.
This process enumerates every valid UCA path permitted by the control structure, ensuring completeness and preventing omissions.
For each enumerated UCA propagation path, a lower-level generation step is then invoked to construct a concrete causal scenario.
Specifically, the model is asked whether a plausible real-world situation exists in which the UCA can occur.
If such a scenario is feasible, it must be explicitly considered by system designers during safety review.
